\documentclass[11pt]{article}

\usepackage{amsmath}
\usepackage{amsthm}
\allowdisplaybreaks
\usepackage{amssymb}
\usepackage{algorithm}
\usepackage[lined,boxed,ruled,norelsize,algo2e,linesnumbered]{algorithm2e}
\usepackage{subfig}

\usepackage{array}
\usepackage{color}
\usepackage[dvipsnames]{xcolor}
\usepackage{wrapfig,epsfig}
\usepackage{epstopdf}
\usepackage{url}
\usepackage{graphicx}
\usepackage{color}
\usepackage{epstopdf}
\usepackage{algpseudocode}
\usepackage{scrextend}
\usepackage[T1]{fontenc}
\usepackage{bbm}
\usepackage{comment}
\usepackage{authblk}
\usepackage{multicol}
\usepackage{dsfont}
\usepackage{mathtools}
\usepackage{enumitem}
\usepackage{mathrsfs}

\usepackage{tikz}
\usepackage{hyperref}  
\hypersetup{colorlinks=true,citecolor=blue,linkcolor=red} 

\usepackage{thmtools}
\usepackage{thm-restate}

\usetikzlibrary{arrows}
\usepackage[margin=1in]{geometry}

\graphicspath{{./figs/}}

\definecolor{b2}{RGB}{51,153,255}
\definecolor{mygreen}{RGB}{80,180,0}

\newcommand{\Daogao}[1]{\textcolor{mygreen}{[Daogao: #1]}}

\theoremstyle{plain}
\newtheorem{theorem}{Theorem}[section]
\newtheorem{lemma}[theorem]{Lemma}
\newtheorem{proposition}[theorem]{Proposition}

\theoremstyle{definition}
\newtheorem{definition}[theorem]{Definition}

\theoremstyle{remark}
\newtheorem{remark}[theorem]{Remark}

\newcommand{\wh}{\widehat}

\newcommand{\ov}{\overline}

\newcommand{\og}{\overline{g}}

\newcommand{\eps}{\varepsilon}
\renewcommand{\epsilon}{\varepsilon}
\renewcommand{\phi}{\varphi}

\newcommand{\R}{\mathbb{R}}

\newcommand{\calZ}{\mathcal{Z}}
\newcommand{\calD}{\mathcal{D}}
\newcommand{\calP}{\mathcal{P}}
\newcommand{\calN}{\mathcal{N}}

\newcommand{\calO}{\mathcal{O}}

\newcommand{\alg}{\mathrm{ALG}}

\newcommand{\hell}{\hat{\ell}}
\newcommand{\hL}{\hat{L}}

\renewcommand{\hat}{\wh}
\renewcommand{\bar}{\ov}
\renewcommand{\d}{\mathrm{d}}

\DeclareMathOperator*{\E}{\mathbb{E}}
\DeclareMathOperator*{\M}{\mathcal{M}}

\DeclareMathOperator{\supp}{supp}

\DeclareMathAlphabet{\mathpzc}{OT1}{pzc}{m}{it}
\newcommand{\nSG}{\mathrm{nSG}}

\newcommand{\ind}{\mathbf{1}}

\newcommand{\Lap}{\mathrm{Lap}}

\providecommand{\minimize}{\mathop{\rm minimize}}

\providecommand{\subjectto}{\mathop{\rm subject\;to}}

\usepackage[capitalize]{cleveref}

\usepackage{cellspace, tabularx}                 
\usepackage{makecell}

\newcolumntype{C}{>{\centering\arraybackslash}X}
\addparagraphcolumntypes{C}
\newcolumntype{P}[1]{>{\arraybackslash}p{#1}}
\usepackage{array}
\usepackage{multirow}

\newcolumntype{x}[1]{%
	>{\raggedleft\hspace{0pt}}p{#1}}%
\usepackage{soul}

\definecolor{darkblue}{rgb}{0,0,.75}

\definecolor{darkred}{rgb}{.6,.1,.14}

\newcommand{\darkblue}[1]{\textcolor{darkblue}{#1}}

\newcommand{\ed}{\ensuremath{(\eps,\delta)}}

\newcommand{\norm}[1]{\left\|{#1}\right\|} 
\newcommand{\lone}[1]{\norm{#1}_1} 

\newcommand{\ha}[1]{
		\textcolor{blue}{\textbf{HA:} {#1}}
}

\newcommand{\reals}{\mathbb{R}} 

\newcommand{\AboTh}{\mathsf{AboveThreshold}} 

\title{User-level Differentially Private Stochastic Convex Optimization: \\ 
Efficient Algorithms with Optimal Rates}
\author{Hilal Asi\thanks{Apple Inc. \texttt{hilal.asi94@gmail.com}} \quad \quad Daogao Liu\thanks{University of Washington. Work done while interning at Apple. \texttt{dgliu@uw.edu} }}
\date{July 2023}

\begin{document}
\maketitle

\begin{abstract}
We study differentially private stochastic convex optimization (DP-SCO) under user-level privacy, where each user may hold multiple data items. Existing work for user-level DP-SCO either requires super-polynomial runtime~\cite{GKKM+23} or requires the number of users to grow polynomially with the dimensionality of the problem with additional strict assumptions~\cite{BS23}. We develop new algorithms for user-level DP-SCO that obtain optimal rates for both convex and strongly convex functions in polynomial time and require the number of users to grow only logarithmically in the dimension. Moreover, our algorithms are the first to obtain optimal rates for non-smooth functions in polynomial time. These algorithms are based on multiple-pass DP-SGD, combined with a novel private mean estimation procedure for concentrated data, which applies an outlier removal step before estimating the mean of the gradients. 

\end{abstract}

\section{Introduction}






Differentially private stochastic convex optimization (DP-SCO) is a central problem in privacy-preserving machine learning, whose aim is to minimize a convex function 
\begin{equation}
  \label{eqn:objective}
  \begin{split}
    \minimize ~ &  L_\calP(\theta) := \E_{z \sim \calP}[\ell(\theta;z)]  \\
    \subjectto ~ & \theta \in \Theta \subset \reals^d,
  \end{split}
\end{equation}

under the constraint of differential privacy, given $n$ users each holding a single sample $z_i \in \calZ$ from the distribution $\calP$. 
Numerous works have studied this problem, known as item-level DP-SCO, and it is by now relatively well understood~\cite{BassilySmTh14,BassilyFeTaTh19,FeldmanKoTa20,AsiFeKoTa21,AsiDuFaJaTa21,kll21}.

A significant concern about item-level DP-SCO in practice is that each user may hold and contribute multiple items to the dataset, significantly degrading the actual privacy protection provided by the item-level differentially private to users. This is the case in many machine learning applications in practice, such as training language and vision models on users' data in federated learning. To address this problem, prior work has studied user-level versions of differential privacy, where the algorithm preserves privacy for users that may contribute $m \ge 1$ items~\cite{LiuTheertaYuKumarRiley20,GhaziKumarManurangsi21,LevySuAmKaKuMoSu21}. This definition is stronger than item-level DP as it forces the algorithm not to be sensitive to changes of a single user or equivalently $m$ items.

Motivated by the realistic and strong privacy protections guaranteed by user-level privacy, many papers have studied DP-SCO under this notion of privacy. \cite{LevySuAmKaKuMoSu21} has initiated the study of this problem and proposed new algorithms based on localized SGD. The main observation in~\cite{LevySuAmKaKuMoSu21}  is that averaging the gradients of users in SGD results in gradients that are concentrated in a ball of small radius of roughly $1/\sqrt{m}$, yielding a final excess risk of $1/\sqrt{nm} + d/n\sqrt{m} \eps$. However, as the optimal rates for item-level DP-SCO ($m=1$) are known to be $1/\sqrt{n} + \sqrt{d}/n \eps$,  it is evident that the rates of~\cite{LevySuAmKaKuMoSu21} are sub-optimal. Moreover, their algorithms are applicable only to smooth functions.

Two recent works of~\cite{BS23,GKKM+23} have resolved some of these issues. \cite{BS23} developed new algorithms based on DP-SGD with improved mean estimation procedures to obtain an optimal rate $1/\sqrt{nm} + \sqrt{d}/n\sqrt{m} \eps$. However, their algorithms also require smoothness of the function and require a stringent lower bound on the number of users $n \ge \sqrt{d}/\eps$.
Moreover, their algorithm cannot work for large $m$ and requires $m\le \max\{\sqrt{d},n\eps^2/\sqrt{d}\}$.
On the other hand, \cite{GKKM+23} observe that user-level DP-SCO has small local sensitivity to deletions, and use propose-test-release to design new algorithms.
Their algorithm requires only $n \le \log(d)/\eps$ users and is also applicable to non-smooth functions. However, it runs in super-polynomial time and achieves sub-optimal error $1/\sqrt{nm} + \sqrt{d}/n\sqrt{m} \eps^{2.5}$.

As a result, existing algorithms for user-level DP-SCO are not satisfactory: they either require smoothness and a large number of users that grow polynomially with the dimension~\cite{BS23}, or run in super-polynomial time~\cite{GKKM+23}.

\subsection{Contributions and Technical Overview}
In this work, we develop new algorithms for user-level DP-SCO that resolve the abovementioned issues. In particular, our algorithms obtain optimal rates in polynomial time, are applicable for non-smooth functions, and require the number of users to grow only logarithmically in the dimension $ n \le \frac{\log(d)}{\eps}$. We summarize our results for the convex case and compare them to prior work in~\Cref{tab:comp}. Additionally, building on our algorithm for the convex case, we propose a new algorithm that obtains optimal rates for user-level DP-SCO in the strongly convex case.


\begin{table*}[t]
\begin{center}
		\begin{tabular}{| Sc | Sc | Sc | Sc |}
		    \hline
			  &  \textbf{\darkblue{Excess Risk}} & \textbf{\darkblue{Polynomial Runtime}}  & \textbf{\darkblue{Number of Users}}  \\
			\hline 
			 \cite{BS23} & $ \frac{1}{\sqrt{nm}} + \frac{\sqrt{d}}{n\sqrt{m}\eps} $ & Yes & $n \ge \frac{\sqrt{d}}{\eps}$ \\
			\cline{1-4} 
			\cite{GKKM+23} & $ \frac{1}{\sqrt{nm}} + \frac{\sqrt{d}}{n\sqrt{m}\eps^{2.5}} $  & No  & $n \ge \frac{1}{\eps}$\\ 
			\cline{1-4}
			\textbf{{This work}} & $ \frac{1}{\sqrt{nm}} + \frac{\sqrt{d}}{n\sqrt{m}\eps} $  & Yes & $n \ge \frac{1}{\eps}$ \\
			\hline
		\end{tabular}
     \end{center}
          \caption{Comparison of excess risk bounds for user-level DP-SCO with prior work, with logarithmic terms omitted. The work of~\cite{BS23} additionally requires smoothness of the loss function and $m\le \max\{\sqrt{d},n\eps^2/\sqrt{d}\}$. 
          }
     \label{tab:comp}
\end{table*}

Our algorithm follows a similar recipe to that of~\cite{BS23}: as it is well known that DP-SGD is optimal in the item-level setting, we wish to extend it to user-level DP using new mean estimation procedures that add less noise to estimate the gradients at each iteration. To this end, note that if we average the gradients of each user using their $m$ samples, this guarantees that the resulting averaged gradients of all users will lie in a ball of radius roughly $\tau = 1/\sqrt{m}$. This concentration allows to design algorithms for mean estimation with sensitivity $\tau/n$ (instead of $1/n$), hence obtaining error (e.g., ~\cite{BS23}) $\tau \sqrt{d}/n \eps_i$ for estimating the gradients at iteration $i$, where $\eps_i$ is the privacy budget at iteration $i$. As we have $T$ iterations, this requires $\eps_i = \eps/\sqrt{T}$. The key challenge here is that private mean estimation procedures for $\tau$-concentrated data (e.g. ~\cite{LevySuAmKaKuMoSu21,BS23}) require $n \ge 1/\eps_i = \sqrt{T}/\eps$, which results in a strong restriction on the number of rounds $T$ that we can run.

Our main challenge is then to design a private mean estimation procedure with $T$ iterations. Each iteration we wish to estimate the mean of $\tau$-concentrated data with privacy budget $\eps_i = \eps/\sqrt{T}$  such that the error at each iteration is $\tau \sqrt{d}/n \eps_i$, and the algorithm uses only $n \le \log(T)/\eps$ samples. We develop a new private mean estimation algorithm for $\tau$-concentrated data that satisfies these properties.

Our approach draws inspiration from the FriendlyCore framework~\cite{TCK+22}, which we use for removing outliers from the dataset. Our methodology has two distinct phases: in the initial stage, we employ an outlier-elimination process that yields a subset of data samples exhibiting $\tau$-concentration. Subsequently, we privatize the mean of the concentrated sample by adding Gaussian noise proportional to $\tau$.

Our outlier-detection phase is based on a score we give to each sample to measure how likely it is to be an outlier; the score measures how many samples in the dataset are in a ball of size $\tau$ around the sample.
We then keep each sample in the dataset with probability proportional to its score, hence removing outliers that have low scores. To guarantee that our final algorithm is private, we have to upper bound the sensitivity of the mean of the sub-sampled dataset is minor. To this end, we apply an extra step via AboveThreshold~\cite{dr14} to verify that the input dataset is nearly $\tau$-concentrated, hence limiting the number of outliers that can be detected.

This improved mean estimation procedure is the building block of all of our results: it allows us to use DP-SGD with a small number of users and run it for large number of rounds to get the optimal rate. Moreover, the large number of rounds made possible by our mean estimation procedure allows us to use randomized smoothing in order to obtain optimal results in the non-smooth case as well, in contrast to prior work where randomized smoothing would not result in optimal rates in the non-smooth setting.

\subsection{Related Work}
User-level differential privacy (DP) is a relatively recent and less-explored area compared to the more established item-level DP setting. It has gained increased attention lately due to its significance in machine learning applications, particularly in the context of federated learning.
Several works have studied user-level DP for several applications, including DP-SCO~\cite{LevySuAmKaKuMoSu21,BS23}, PAC learning~\cite{GhaziKumarManurangsi21}, and discrete distribution estimation~\cite{LiuTheertaYuKumarRiley20,AcharyaLiSu23}.
In recent work, \cite{ghazi2023user} proposed a generic transformation of any item-level DP algorithm to a user-level DP algorithm. However, it is inefficient, and the dependence on $\eps$ may not be optimal.

DP-SCO has been studied in the item-level DP setting extensively~\cite{BassilyFeTaTh19,FeldmanKoTa20,AsiFeKoTa21,AsiDuFaJaTa21,kll21,gopi2023private}. The rates of DP-SCO in the item-level setting are well understood and~\cite{BassilyFeTaTh19} obtained the optimal $1/\sqrt{n} + \sqrt{d \log(1/\delta)}/n\eps$ rate using stability based analysis of DP-SGD with a large batch size. These algorithms are not efficient, leading ~\cite{FeldmanKoTa20} to develop new optimal algorithms for the smooth case that run in linear time. However, the best runtime for the non-smooth setting is super-linear, and this is an ongoing research direction which is still open~\cite{AsiFeKoTa21,kll21,carmon2023resqueing}. Item-level DP-SCO has also been studied in various other settings, such as the stronger pure DP model~\cite{AsiLeDu21}, heavy-tailed data distributions~\cite{lowy2023private}, non-euclidean geometries~\cite{AsiFeKoTa21,BassilyGuNa21},  and non-convex loss functions~\cite{ganesh2023private,arora2023faster}.

\section{Preliminaries}
Let $[k]=\{1,\cdots,k\}$ be the set of positive integers no larger than $k$. Throughout the paper, 
we assume that the loss function $\ell(:,z):\Theta\to \R$ is convex and $G$-Lipschitz for any $z\in \calZ$, and $\Theta\subset\R^d$ is a closed convex domain of diameter $R$.
There are $n$ users, each holding $m$ i.i.d. samples from the underlying distribution $\calP$; we denote the samples of the $i$-th user by  $Z_i=\{z_{i,j}\}_{j\in[m]}$.
We use capital $Z$ to denote one user and $z$ to denote one item.
The dataset $\calD=\{Z_i\}_{i\in[n]}$ contains all the users along with all the items.

The objective is to design efficient algorithms for minimizing $L_\calP(\theta):=\E_{z\sim \calP}\ell(\theta,z)$, which is differentially private at the user level.
For a user $Z_i =\{z_{i,j}\}_{j\in[m]} $, we let $\nabla L(\theta;Z_i):=\frac{1}{m}\sum_{j\in[m]}\nabla\ell(\theta;Z_{i,j})$ denote the average of the gradients for the user's samples.
We denote the empirical function $L_\calD(\theta):=\frac{1}{nm}\sum_{z\in Z_i}\sum_{Z_i\in\calD}\ell(\theta,z)$. For a distribution $X$, we let $\supp(X)$ be the support of the distribution $X$.

\subsection{Differential Privacy}
In this work, we use the notion of user-level differential privacy where each user has a sample $z \in \calZ^m$.
\begin{definition}[User-Level Differential Privacy]
    A mechanism $\M : (\calZ^m)^n \to \reals^d$ is $(\epsilon,\delta)$ user-level differentially private, if for any neighboring datasets $\calD,\calD' \in (\calZ^m)^n$ that differ in one user, and for any event $O$ in the range of $\M$, we have
    \begin{align*}
    \Pr[M(\calD)\in O]\le e^{\eps}\Pr[M(\calD')\in O]+\delta.
    \end{align*}
\end{definition}
Note that item-level differential privacy is a specific case of this definition where $m=1$.

Additionally, our analysis requires the notion of indistinguishability between two random variables.
\begin{definition}[Indistinguishablity]
    Two random variables $X$ and $Y$ are $(\epsilon,\delta)$-Indistinguishable if for any event $\calO$, we have
    \begin{align*}
        \Pr[X\in\calO]\le e^\epsilon \Pr[Y\in\calO]+\delta,\\
        \text{ and }\Pr[Y\in\calO]\le e^\epsilon \Pr[X\in\calO]+\delta.
    \end{align*}
\end{definition}

Moreover, for two distributions $X$ and $Y$, we use the notation $X\sim_\gamma Y$ to denote that the total variation distance between $X$ and $Y$ is bounded by $\gamma$.
We also define the following divergence.
\begin{definition}
Given two distributions $X$ and $Y$, the $\delta$-approximate max divergence between $X$ and $Y$ is defined as
\begin{align*}
    D_\infty^\delta(X\|Y)=\sup_{Z\in\supp(X):\Pr[X\in Z]\ge \delta}\log\frac{\Pr[X\in Z]-\delta}{\Pr[Y\in Z]}
\end{align*}
\end{definition}

\subsubsection{AboveThreshold}
Our algorithms use the AboveThreshold algorithm~\cite{dr14} which is a key tool in differential privacy to identify whether there is a query $q_i: \calZ \to \reals$ in a stream of queries $q_1,\dots,q_T$ that is above a certain threshold $\Delta$. 
The $\AboTh$ algorithm (presented in appendix) has the following guarantees.

\begin{algorithm2e}
\caption{$\AboTh$}
\label{alg:mean_est_with_AT}
{\bf Input:} Dataset $\calD = (Z_1,\dots,Z_n) $, threshold $\Delta \in \reals$, privacy parameter $\epsilon$\;
Let $\hat{\Delta}:= \Delta-\Lap(\frac{2}{\epsilon})$\;
\For{$t=1$ to $T$}
{
Receive a new query $q_t: \calZ^n \to \reals$ \;
Sample $\nu_i \sim \Lap(\frac{4}{\epsilon})$\;
\If{$q_t(\calD)+\nu_i<\hat{\Delta}$}
{
{\bf Output:} $a_i=\bot$\;
{\bf Halt}\;
\Else{
{\bf Output:} $a_i=\top$\;
}
}
}
\end{algorithm2e}

\begin{lemma}[\cite{dr14}, Theorem 3.24]
\label{thm:Above_Threshold}
    $\AboTh$ is $(\epsilon,0)$-DP.
    Moreover, let $\alpha=\frac{8\log(2T/\gamma)}{\epsilon}$ and $\calD \in \calZ^n$. For any sequence of $T$ queries $q_1,\cdots,q_T : \calZ^n \to \reals$ each of sensitivity $1$, $\AboTh$ halts at time $k \in [T+1]$ such that with probability at least $1-\gamma$,
    \begin{itemize}
        \item For all $t < k$, $a_t =\top$ and $q_t(\calD) \ge \Delta - \alpha$;
        \item $a_k = \bot$ and $q_k(\calD) \le \Delta + \alpha$ or $k = T+1$.
    \end{itemize} 
\end{lemma}

\subsection{Randomized Smoothing}

To develop optimal algorithms in the non-smooth setting, our algorithm use randomized smoothing~\cite{yns11,dbw12} to make the functions smooth.
To this end, for a convex function $\ell(:;Z)$, we denote the convolution function $\hell(:;Z):=\ell(:;Z)* n_r$., where $n_r$ is the uniform density in the $\ell_2$ ball of radius $r$ centered at the origin in $\R^d$.
Specifically, $n_r(y)= \frac{\Gamma(\frac{d}{2}+1)}{\pi^{\frac{d}{2}}r^d}$ for $\|y\|\le r$, and $n_r(y)=0$ otherwise.
For simplicity, we may omit the dependence on $z$, and write the function as $\hell$ and $\ell$.
Denote $\hL_\calP(\theta):=\E_{z\sim P,y\sim n_r}\ell(\theta+y;z)$ and $\hL_\calD(\theta):=\frac{1}{|\calD|}\sum_{z\in\calD}\E_{y\sim n_r}\ell(\theta+y;z)$.

\begin{lemma}[Randomized Smoothing, \cite{yns11,dbw12}]
\label{lm:random_smooth}
The convolution function has the following properties:
\begin{itemize}
    \item $\hell(\theta)\le \ell(\theta)\le \hell(\theta)+Gr$.
    \item $\hell$ is $G$-Lipschitz and convex.
    \item $\hell$ is $\frac{G\sqrt{d}}{r}$-smooth.
    \item For random variables $y\sim n_r$, and $z\in \calD$, we have 
     $   \E_{y,z}[\nabla \ell(\theta+y;z)]=\nabla \hat{L}_\calD(\theta).$
\end{itemize}
\end{lemma}

\subsection{Norm-Subgaussian Concentration}
Our analysis also uses a notion named concentration properties for norm-Subgaussian random variables.

\begin{definition}[norm-Subgaussian]
A random vector $X\in\R^d$ is norm-SubGaussian with parameter $\sigma$, denoted  $\nSG(\sigma)$, if for all $t \in \reals$
\begin{align*}
    \Pr[\|X-\E X\|\ge t]\le 2\exp(-\frac{t^2}{2\sigma^2}).
\end{align*}
\end{definition}

The following concentration result holds for norm-Subgaussian random variables.
\begin{lemma}[\cite{jngkj19}, concentration of NormSubgaussian]
\label{lem:concen_nSG}
    There exists a constant $c>0$, such that for zero-mean independent random vectors $X_1,\cdots,X_n\in\R^d$ where $X_i$ is $\nSG(\sigma_i)$ for all $i\in[n]$, for any $\delta>0$, with probability at least $1-\delta$,
    \begin{align*}
        \|\sum_{i\in[n]}X_i\|\le c\sqrt{\sum_{i\in[n]}\sigma_i^2\log\frac{2d}{\delta}}.
    \end{align*}
\end{lemma}

We also use the following standard Chernoff-Hoeffding bound.
\begin{lemma}[Chernoff-Hoeffding Bound]
\label{lm:Chernoff}
Let $X_1,\cdots,X_n$ be independent Bernoulli random variables such that $\E[X_i]=p_i$.
Let $X=\sum_{i\in [n]}X_i$ and $\mu=\E[X]$.
Then we know for any $\lambda>0$, we have
\begin{align*}
    \Pr[X\ge (1+\lambda)\mu]\le \exp(-\frac{\lambda^2\mu}{2+\lambda}).
\end{align*}
\end{lemma}

\section{Adaptive Mean Estimation for Concentrated Samples}
\label{sec:mean_est}
The main component in our algorithms is a novel mean estimation procedure for adaptive queries for $\tau$-concentrated samples where the samples lie in a ball of radius $\tau$ (see Definition~\ref{def:concent}). This algorithm will be used to estimate the gradients in our optimization procedure, as the user-level setting will guarantee that $\tau \approx  1/\sqrt{m}$ for an i.i.d. input.
We add Gaussian noise scales with $\tau$; hence, the final loss bound benefits from small $\tau$.

\begin{definition}
\label{def:concent}
A random samples $\{X_i\}_{i\in[n]}$ is $(\tau,\gamma)$-concentrated if there exists a point $x\in \R^d$ such that with probability at least $1-\gamma$,
\begin{align*}
    \max_{i\in[n]}\|X_i-x\|\le \tau.
\end{align*}
\end{definition}

Given $T$ adaptive mean estimation queries $q_1,\dots,q_T : (\calZ^m)^n \to \reals^d$ such that the $n$ users are $\tau$-concentrated with respect to these queries, our goal is to get a nearly unbiased estimate of the mean of each query with variance $\frac{\tau^2 Td}{n^2\epsilon^2}$ under \ed-DP. 
The standard approach for solving this task, as done in~\cite{BS23}, is to assign a privacy budget $\eps_i = \eps/\sqrt{T}$ for each query, hence resulting in variance $\frac{\tau^2 Td}{n^2\epsilon^2}$. 
However, this procedure requires $n \ge \frac{1}{\eps_i}=\sqrt{T}/\eps$ to guarantee the desired utility bounds, which is too prohibitive for our purposes.

In this section, we design a new algorithm for adaptive mean estimation that achieves the desired variance with only $n \ge 1/\eps$. Our algorithm is inspired by the FriendlyCore framework \cite{TCK+22}, where we use the basic filter to identify outliers in the dataset. Our procedure consists of two stages: first, we apply an outlier-removal procedure, which returns a subset of the samples that is $\tau$-concentrated. Then, we add Gaussian noise proportional to $\tau$ to privatize the mean of the concentrated sample.

To identify outliers, we give a score to each sample, which measures how many samples in the dataset are in a ball of size $\tau$ around the sample. As outlier samples will have a low score, we then keep each sample in the dataset with probability proportional to its score. This will preserve privacy for samples that are nearly $\tau$-concentrated, whereas we aim to preserve privacy for all input datasets. Therefore, we add an initial check to the algorithm which verifies that the algorithm is nearly $\tau$-concentrated. To this end, we define a $\tau$-concentration score of the dataset for a query $q_i$ to be 
\newcommand{\qcon}{s^{\mathsf{conc}}}
\begin{align}
\label{eq:friendly_query}
    \qcon_i(\calD,\tau):=\frac{1}{n}\sum_{z\in\calD}\sum_{z'\in \calD}\ind(\|q_i(z)-q_i(z')\|\le \tau).
\end{align}
and check via $\AboTh$ that this score is above the desired threshold for all queries. 
The following procedure will be processed only if the dataset and the queries pass the check, which means our samples are nearly concentrated and ensures the privacy guarantee of the following procedure.
We describe the full details of our algorithm in~\Cref{alg:gauss_mech_gradient}.




\begin{algorithm2e}[h]
\caption{Outlier-Removal Based Mean Estimation for Concentrated Data}
\label{alg:gauss_mech_gradient}
{\bf Input:} Dataset $\calD = (Z_1,\dots,Z_n) $, privacy parameters $(\epsilon,\delta)$,  parameters $\tau$  \;
\For{$i=1$ to $T$}
{
Receive a new mean estimation query $q_i: \calZ \to \reals^d$ \;
Define concentration score $$\qcon_i(\calD,\tau):=\frac{1}{n}\sum_{Z\in\calD}\sum_{Z'\in \calD}\ind(\|q_i(Z)-q_i(Z')\|\le \tau)\;$$ 
\If{$\AboTh(\qcon_i, \epsilon/2, 4n/5)=\top$ }
{
Set $S_i=\emptyset$\;
\For{Each User $Z_j\in\calD$}
{

Set $f_{i,j}=\sum_{Z\in\calD}\ind(\| q_i(Z_j)- q_i(Z)\|\le2\tau)$\;
Add $Z_j$ to $S_i$ with probability $p_{i,j}$ for 
$
p_{i,j}=
\begin{cases}
    0 & f_{i,j}< n/2 \\
    1 & f_{i,j}\ge 2n/3\\
    \frac{f_{i,j}-n/2}{n/6} & o.w.
\end{cases}  
$
}
Let $g_i=\frac{1}{|S_i|}\sum_{Z\in S_i}q_i(Z)$ if $S_i$ is not empty, and $0$ otherwise  \;
{\bf Output}: $\hat g_i\leftarrow g_i+\nu_i$, where $\nu_i\sim\calN(0,\frac{8\tau^2T\log(e^{\epsilon}T/\delta)\log(e^{\epsilon/2}/\delta)}{n^2\epsilon^2}I_d)$ 
}
\Else
{
{\bf Output:} $g_i=\mathbf{0}$\; 
{\bf Halt}\;
}
}
\end{algorithm2e}

The following theorem summarizes the main guarantees of our algorithm.
\begin{theorem}
\label{thm:mean-est}
For $0<\epsilon<10,0<\delta<1$.
Let $\calD = (Z_1,\dots,Z_n) \in (\calZ^m)^n$ be a dataset with $n \ge \frac{8\log(T/\gamma)+8\log(T/\delta)}{\eps}$ users. 
\Cref{alg:gauss_mech_gradient} is \ed-DP. Moreover, if $(q_i(Z_1),\dots,q_i(Z_n))$ is $(\tau,\gamma)$-concentrated for all $i \in [T]$ and let  $\{\hat g_i\}_{i\in[T]}$ be the outputs of \Cref{alg:gauss_mech_gradient}, then there exists random variables $\hat g_1',\dots,\hat g_T'$ such that the joint distributions $\{\hat g_i\}_{i\in[T]}\sim_{(1+T)\gamma} \{\hat g_i'\}_{i\in[T]}$.
Moreover, for each $i\in[T]$, given $\{g_j'\}_{j\le i-1}$ and $q_i$, $\hat g_i'$ satisfy that  
\begin{align*}
    &\E \hat g_i'=\frac{1}{n}\sum_{j=1}^{n} q_i(Z_j),\\
 &	\E\norm{\hat g_i' - \frac{1}{n} \sum_{j=1}^n q_i(Z_j) }^2 \lesssim \frac{\tau^2 T \log(T/\delta)\log(1/\delta)}{n^2 \eps^2}.
\end{align*}
\end{theorem}

To prove~\Cref{thm:mean-est}, we consider the privacy and utility guarantees separately.
We argue about privacy first.
The following lemma upper bounds the sensitivity of the probability distribution $p_{i}$ for adding users to $S_i$.
\begin{restatable}{lemma}{boundlonef}
\label{lem:bound_l1_f}
For any neighboring dataset $\calD,\calD'$ that differs in one user, let $p_i=(p_{i,1},\cdots,p_{i,n})$  be the probability for users to be selected into $S_i$ for $\calD$, and let $p_i'$ be the corresponding probability for $\calD'$.
Then
\begin{align*}
    \|p_i-p_i'\|_1\le 2.
\end{align*}
\end{restatable}


In the following lemma, we show when $|p_i - p_i'| \le 2$, then the hamming distance between the selected sets $S_i$ and $S_i'$ cannot be large. Thus,
given the low sensitivity of $p_i$ for two neighboring datasets, this shows that sub-sampled datasets at round $i$ will not be too far from each other.

\begin{restatable}{lemma}{couplingovermultiber}
    \label{lem:coupling_over_multi_ber}
Let $p,p'\in[0,1]^n$ such that $\|p-p'\|_1\le 2$, and let $V$ and $V'$ be drawn from $\mathrm{Ber}(p)$ and $\mathrm{Ber}(p')$ respectively.
For any $\zeta\in(0,1)$, there exists a coupling $\Gamma$ over $V$ and $V'$ such that for $(x,y)$ drawn from $\Gamma$, with probability at least $1-\zeta$,
\begin{align*}
    \|x-y\|_1\le O(\log(1/\zeta)).
\end{align*}
\end{restatable}


Now, we analyze the privacy guarantee.
Since $\AboTh$ is private, it suffices to prove privacy for the case where $\AboTh$ always outputs ``$\top$'', as otherwise the output is $\bf 0$. Note that when $\AboTh$ outputs  ``$\top$'', the dataset is well concentrated with respect to the queries. This concentration, together with the fact that the sub-sampled datasets are not too far from each other (\Cref{lem:coupling_over_multi_ber}), allows us to upper bound the sensitivity of the mean of the sub-sampled datasets, that is $g_i$.
Hence, the privacy guarantee of the outputs $\{\hat g_i\}$ will follow from the guarantees of the Gaussian mechanism.

To formalize the above intuition, let $a_i\in\{\top,\bot\}$ be the output of $\AboTh$ for $i$-th query.
Recall that in Algorithm~\ref{alg:mean_est_with_AT}, we draw one random variable from $\Lap(\frac{2}{\epsilon})$ and $T$ independent random variables from $\Lap(\frac{4}{\epsilon})$.
Let $E$ be the event that the absolute values of these random variables are no more than $\frac{4\log(2T/\zeta)}{\epsilon}$. Then, we know the probability of $E$ is at least $1-\zeta/2$.
Conditional on $E$, for all $a_i=\top$, we have $q_i\ge\frac{4n}{5}-\alpha$  and for all $a_i=\bot$ we have $q_i\le \frac{4n}{5}+\alpha$. Note that
$\frac{4n}{5}-\alpha \ge \frac{2n}{3}$ by the value of $\alpha$ and the precondition that $n\ge\frac{40\log(2T/\zeta)}{\epsilon}$.
The guarantees of $\AboTh$ (\Cref{thm:Above_Threshold}) also imply that the measure of $E$ is at least $1-\zeta$.
Define $E'$ to be the event w.r.t. input $\calD'$. 

The following lemma upper bounds the sensitivity of the mean of the sub-sampled datasets.

\begin{restatable}{lemma}{couplingclosegi}
    \label{lem:coupling_close_gi}
For any $i$-th iteration and any neighboring datasets $\calD,\calD'$, conditional on $E$ and $E'$ and conditional on $a_i=a_i'$, there exists a coupling $\Gamma_i$ over $g_i$ and $g_i'$, such that for $(x,y)$ drawn from $\Gamma_i$, with probability at least $1-\zeta$,
\begin{align*}
    \|x-y\|_2\lesssim \frac{\tau\log(1/\zeta)}{n}.
\end{align*}
\end{restatable}


Given the sensitivity bound of~\Cref{lem:coupling_close_gi}, we can argue for indistinguishability of the outputs using advanced composition and standard guarantees of the Gaussian mechanism.

\begin{restatable}{proposition}{privacymeanest}
    \label{lem:privacy_mean_est}
    For any dataset $\calD$,
if $n\ge \frac{40\log(4T/\delta)}{\epsilon}$, then for any neighboring dataset $\calD'$, the outputs of Algorithm~\ref{alg:gauss_mech_gradient} with $\calD$ and $\calD'$ as inputs are $(\epsilon, \delta)$-indistinguishable.
\end{restatable}
\begin{proof}(sketch)
    We only provide a sketch of the proof here and defer the full proof to~\Cref{sec:proof-priv-mean}. First, note that $a_1,\dots,a_T \in \{ \top,\bot \}$ are $\eps/2$-DP using the guarantees of $\AboTh$. Moreover, if there exists an $a_i = \bot$ then $g_i$ is post-processing of $a_i$ hence private as well. Thus, we prove privacy of $\{\hat g_1, \dots, \hat g_T \}$ assuming $a_1=a_2=\dots=a_T = \top$.
    First, we condition on the high-probability event $E$ which indicates the success of $\AboTh$ (the failure probability will be added to the $\delta$ term). Under this event, \Cref{lem:coupling_close_gi} implies that the sensitivity of $g_i$ is bounded by $\frac{\tau\log(1/\zeta)}{n}$. Thus, advanced composition and the guarantees of the Gaussian mechanism imply that $\{ \hat g_1,\dots,\hat g_T \}$ are $(\eps/2,\delta)$-DP. The claim follows.
\end{proof}


Having estabilished the privacy guarantee of~\Cref{alg:gauss_mech_gradient}, we now proceed to prove its utility.
The following proposition shows that if the dataset is well concentrated with respect to the query, then no user will be removed in the outlier-removal stage with high probability, hence the estimate is nearly unbiased.

\begin{restatable}{proposition}{unbiased}
\label{lem:unbiased}
For all $i\in[T]$, if $(q_i(Z_1),\dots,q_i(Z_n))$ is $(\tau,\gamma)$-concentrated and  $n\ge \frac{8\log (T/\gamma)}{\epsilon}$,
    then with probability at least $1-(T+1)\gamma$, we have $S_i=\calD$ for all $i\in[T]$. In particular, it holds that $g_i = \frac{1}{n} \sum_{Z \in \calD} q_i(Z)$ with probability at least $1-(T+1)\gamma$.
\end{restatable}

\begin{proof}
To prove the lemma, we have to show that $\AboTh$ will succeed (output $\top$) for each $i \in [T]$, and that the outlier-removal stage will not remove any item from the set. This will imply that $S_i=\calD$ for all $i\in[T]$, hence $g_i = \frac{1}{n} \sum_{Z \in \calD} q_i(Z)$.

To this end, fix any $i\in [T]$.
Under the precondition that $(q_i(Z_1),\dots,q_i(Z_n))$ is $(\tau,\gamma)$-concentrated, we know that $\qcon_i(\calD,\tau)=n$ with probability $1-\gamma$ for each $i \in [T]$. Moreover, the guarantees of $\AboTh$ (\Cref{thm:Above_Threshold}) imply that it will output ``$\top$'' with probability at least $1-\gamma/T$ for each $i \in [T]$ when $\qcon_i(\calD,\tau)=n$.
Finally,  under the event that $(q_i(Z_1),\dots,q_i(Z_n))$ is $\tau$-concentrated, we have that $f_{i,j}=n$ for each user $Z_j\in\calD$, and hence $Z_j$ will be added into $S_i$.
The statement follows by applying a union bound.
\end{proof}


The utility guarantees of Theorem~\ref{thm:mean-est} now follows from~\Cref{lem:unbiased} by setting $\hat g_i' = \frac{1}{n} \sum_{Z \in \calD} q_i(Z) +\nu_i$, where $\nu_i\sim\calN(0,\frac{8\tau^2T\log(e^{\epsilon}T/\delta)\log(e^{\epsilon/2}/\delta)}{n^2\epsilon^2}I_d)$. 


\section{Optimal Rates for User-Level DP-SCO}
In this section, we present our main algorithm for user-level DP-SCO based on the gradient estimation procedure constructed above.
Our algorithm leverages the Stochastic Gradient Descent (SGD) over a smoothed version of the loss function using randomized smoothing by applying the gradient estimation procedure to get (nearly) unbiased stochastic gradients. We present the full details of the algorithm in~\Cref{alg:DPsgd}.

Three key techniques are crucial for our algorithm and its analysis: first, for a fixed $\theta \in \Theta$, a simple concentration argument shows that the average gradient of each user will lie with high probability in a ball of small radius around the population gradient (see Lemma~\ref{lem:concentrated_grd})
\begin{equation*}
    \|\nabla L(\theta;Z_i)-\nabla L_\calP(\theta)\|\le \frac{G\log(nd/\gamma)}{\sqrt{m}}.
\end{equation*}
This is not sufficient for our algorithms as we need this property to hold for data-dependent $\theta_t$. To this end, similarly to~\cite{BS23}, we use the generalization properties of differential privacy to show in~\Cref{lem:concentrated_theta_t} that a similar concentration holds for $\nabla L(\theta_t;Z_i)$. Given this concentration, our mean estimation procedure (\Cref{alg:gauss_mech_gradient}) adds lower noise to estimate of the gradients.

Our second technique is based on the observation that smoothness is necessary to obtain the full potential of DP-SGD in user-level DP-SCO (similarly to existing work that used SGD-based algorithms for user-level DP-SCO~\cite{LevySuAmKaKuMoSu21,BS23}). 
Convergence rates of SGD cause the limitation for non-smooth functions, which depend on the second moment of the gradients, whereas it depends on the variance for smooth functions (\Cref{prop:SGD}). As averaging the gradients of $m$ samples reduces the variance while keeping the second moment the same, this yields better performance for smooth functions. To address this, we adopt randomized smoothing to smooth the loss functions and apply SGD over the smoothed functions. This is made possible due to our mean estimation procedure, which only requires $n \ge \log(mnd/\delta)/\eps$, in contrast to prior work, which required $n \ge \sqrt{T}/\eps$; this strict bound on the number of rounds is not sufficient to obtain optimal rates with randomized smoothing.

Finally, as we are using multi-pass SGD, an additional argument is needed to guarantee a low risk for population error. To this end,  
we analyze the stability of our algorithm for non-smooth functions using~\cite{bfgt20}, which implies that our algorithm has low generalization error.

\begin{algorithm2e}
\caption{DP-SGD for user-level DP}
\label{alg:DPsgd}
{\bf Input:} Dataset $\calD = (Z_1,\dots,Z_n) \in (\calZ^m)^n$, private parameters $(\eps,\delta)$, initial point $\theta_0$, convolution parameter $r$,
number of rounds $T$, stepsize $\eta$, concentration parameter $\tau$, initial distance $\hat R$\;
\For{$t=1,\cdots,T$}
{
Define a query $q_t(Z) = \frac{1}{m} \sum_{j=1}^m \nabla \hat{\ell}(\theta_t;z_{i,j})$ for $Z \in \calZ^m$, See Equation~\eqref{eq:gradient_hell} for the definition \; 
Run~\Cref{alg:gauss_mech_gradient} with query $q_t$ and parameters $\calD,\epsilon,\frac{\delta}{2Tmnd},\tau$ \;
Let $\bar g_t$ be the output of~\Cref{alg:gauss_mech_gradient} \;
\If{$\bar g_t \neq \bot$}
{
Update  $\theta_{t+1}\leftarrow \Pi(\theta_{t}-\eta \overline{g}_t)$\;
}
\Else
{
{\bf Output:} Initial point $\theta_0$\;
{\bf Halt}
}

}
{\bf Return:} $\hat \theta = \frac{1}{T}\sum_{t\in[T]}\theta_t$
\end{algorithm2e}

Let $\Theta_r=\{\theta+y:\theta\in\Theta, \|y\|\le r\}$.
The following theorem summarizes our main result.

\begin{restatable}[User-level DP-SCO]{theorem}{generalconvex}
    \label{thm:general_convex}
Let $0<\epsilon<10$ and $0<\delta<1$.
Algorithm~\ref{alg:DPsgd} is user-level $(\epsilon,\delta)$-DP.
Setting $\hat R= R, r=\frac{d^{1/4}\hat R}{\sqrt{T}},\eta=\frac{\hat R}{G}\cdot\min\{\frac{\sqrt{m}n\epsilon}{T\sqrt{d\log^2(mnd/\delta)}},\frac{1}{T^{3/4}},\frac{\sqrt{nm}}{T}\}$, $\tau= \frac{G\log(ndme^\epsilon T/\delta)}{\sqrt{m}}$ and $T=O(m^2n^2+mn\sqrt{d})$, 
if  $\Theta\subset \R^d$ is a convex set of diameter $R$, $\{\ell(:,z)\}_{z\in \calZ}$ is a family of $G$-Lipschitz convex function over $\Theta_r$, each item in $\calD$ is drawn i.i.d. from the underlying distribution $P$, and $n\gtrsim\frac{\log (mdn/\delta)}{\epsilon}$,
then the output $\hat \theta$  of Algorithm~\ref{alg:DPsgd} satisfies
\begin{align*}
& \E\left[ L_{\calP}(\hat \theta) - \min_{\theta^\star \in \Theta} L_{\calP}( \theta^\star) \right] \le
    O\left(GR \cdot \left(\frac{1}{\sqrt{nm}}+\frac{\sqrt{d\log^2(ndm/\delta)}}{n \sqrt{m}\epsilon} \right) \right).
\end{align*}
\end{restatable}


\begin{remark}
If we have a random initial point $\theta_0$ such that $\E[ \|\theta_0-\theta^*\|^2]\le R'^2$ for $\theta^*=\arg\min L_{\calD}(\theta)$ and some $R'<R$, then we can replace the parameter setting $\hat R= R$ by $\hat R= R'$ in the population loss bound and the dependence on $R$ can be reduced to $R'$ in the loss bound.
\end{remark}

\begin{remark}
    We define the functions on $\Theta_r$ rather than $\Theta$ to make use of the randomized smoothing technique. As $r$ is much smaller than $R$, this impact can be minimal.
    One can eliminate this domain extension by applying other smoothing techniques, such as the Moreau envelope smoothing method, but this method will increase the gradient computation cost.
\end{remark}

In some regimes, when $d$ is large, we can set $T$ in the theorem statement smaller. But we omit those terms to avoid complexity, as getting smaller $T$ is not the primary goal of this work.

We begin by showing that the gradients are concentrated.
For any user $Z_i$ who holds $m$ items denoted by $\{z_{i,j}\}_{j\in[m]}$ and any point $\theta\in \Theta$, we denote 
\begin{align}
\label{eq:gradient_hell}
    \nabla\hell(\theta;Z_i):=\frac{1}{m}\sum_{j\in[m]}\nabla\ell(\theta+y_j;z_{i,j}),
\end{align}
the average stochastic gradients of all items owned by $Z_i$, where $y_j\sim n_r$ is drawn independently of $\theta$ and $z_{i,j}$ for the randomized smoothing.

Our goal is to eventually prove that $\{\nabla\hell(\theta_t;Z_i)\}_{Z_i\in \calD}$ are concentrated. To this end, we start with proving  concentration for $\{\nabla\hell(\theta;Z_i)\}_{Z_i\in \calD}$ for a fixed $\theta \in \Theta$.

\begin{restatable}{lemma}{concentratedgrd}
\label{lem:concentrated_grd}
For any fixed $\theta$ and for each $Z_i$, if each item in $Z_i$ is drawn i.i.d. from $\calP$, with probability at least $1-\gamma/n$, we have
\begin{align*}
    \|\nabla\hell(\theta;Z_i)-\nabla\hL_\calP(\theta)\|\le \frac{G\log(nd/\gamma)}{\sqrt{m}},
\end{align*}
\end{restatable}


One issue with applying Lemma~\ref{lem:concentrated_grd} to demonstrate the concentration property of the stochastic gradients is that the dataset $\calD$ and the points $\{\theta_i\}_{i\in[T]}$ are not independent.
To tackle this, similarly to~\cite{BS23}, we make use of the generalization properties of private mechanisms. We need the following lemma.

\begin{lemma}[Lemma~3.7 in \cite{fmt22}]
\label{lem:approximateDP_to_pureDP}
Let $\alg$ be an $(\epsilon,\delta)$-DP algorithm with respect the input $\calD$.
Then there exists an $(2\epsilon,0)$-DP algorithm $\alg'$, such that
\begin{align*}
    d_{TV}(\alg(\calD),\alg'(\calD))\le \delta.
\end{align*}
\end{lemma}

\begin{restatable}[Similar to Theorem 3.4 in \cite{BS23}]{lemma}{concentratedthetat}
\label{lem:concentrated_theta_t}
    Suppose $\calD=\{z_{i,j}\}_{i\in[n],j\in[m]}$ are drawn i.i.d. from the distribution $\calP$.
In Algorithm~\ref{alg:DPsgd}, for all $t\in[T]$, $\{\nabla\hell(\theta_t;Z_i)\}_{Z_i\in \calD}$ is $(\tau,\gamma')$-concentrated for
\begin{align*}
    \tau=\frac{G\log(nd/\gamma)}{\sqrt{m}}, \gamma'=T(e^{2\epsilon}\gamma+\frac{\delta}{2Tmnd}).
\end{align*}
\end{restatable}


Having established the concentration property of $\{\nabla\hell(\theta_t;Z_i)\}_{Z_i\in \calD}$, we can bound the utility of our procedure for the empirical function $\hL_\calD$. 
Now, we turn to prove the upper bounds for the generalization error, which needs the following well-known Lemma.

\begin{lemma}[\cite{be02}]
    \label{lm:generalization_stability}
For an algorithm $\alg$, a dataset $\calD=\{z_{i,j}\}_{i\in[n],j\in[m]}$ drawn i.i.d. from the distribution $\calP$.
If we replace one random data $z_{i,j}$ in $\calD$ by a fresh new sample $z_{i,j}'$ from $\calP$ and get the dataset $\calD'$ and let $\alg(\calD)$ be the (random) output of the algorithm, one has
\begin{align*}
    & ~\E_{\calD,\alg}\big[L_{\calP}(\alg(\calD))-L_{\calD}(\alg(\calD))\big]\\
    &=\E_{\calD,z_{i,j}',\alg}\big[\ell(\alg(\calD);z_{i,j}')-\ell(\alg(\calD');z_{i,j}') \big].
\end{align*}
\end{lemma}

As we are considering Lipschitz functions, if we can bound the total variation distance between $\alg(\calD)$ and $\alg(\calD')$ where $\calD$ and $\calD'$ differs from one single item, named by algorithmic stability, then we can bound the generalization error.
Formally, we define the algorithmic stability of $\alg$ as follows:
\begin{align*}
    \Lambda(\alg):=d_{TV}(\alg(\calD),\alg(\calD')),
\end{align*}
where $d_{TV}(\alg(\calD),\alg(\calD'))$ denotes the total variation distance between $\alg(\calD)$ and $\alg(\calD')$.
Notably, the user-level differential privacy concerns replace $m$ data of one user, while the algorithmic stability only concerns replacing one single item of a user.
We have the following Lemma.

\begin{lemma}[Lemma 3.1 in \cite{bfgt20}]
\label{lemma:sgd-stab}
Let $(x^t)_{t\in[T]}$ and $(y^t)_{t\in[T]}$ be two trajectories of running SGD for $G$-Lipschitz convex function $f$, that is $x^t=\Pi(x^{t-1}-\eta\nabla f(x^{t-1}))$ and $y^t=\Pi(y^{t-1}-\eta\nabla f'(y^{t-1}))$. Suppose $\|\nabla f(x^t)-\nabla f'(x^t)\|\le a_t\le 2G$ for all $t\in[T]$, then
\begin{align*}
    \|x^T-y^T\|\le 2G\sqrt{\sum_{t\in[T-1]}\eta_t^2}+2\sum_{t\in[T-1]}\eta_ta_t.
\end{align*}
\end{lemma}

We use $\alg$ to represent Algorithm~\ref{alg:DPsgd}.
Then, we can bound the algorithmic stability of $\alg$ based on the unbiased property of our mean estimate procedure (Lemma~\ref{lem:unbiased}) constructed in the previous section.

\begin{restatable}[Algorithmic stability bound]{lemma}{stability}
\label{lem:stability}
    Suppose $\{Z_i\}$ are drawn i.i.d. from the underlying distribution $\calP$.
Suppose $\tau\ge \frac{G\log(ndme^\epsilon T/\delta)}{\sqrt{m}}$ and $n\gtrsim \frac{\log(mdn/\delta)}{\epsilon}$, with probability at least $1-\frac{\delta}{mnd}$, the stability of Algorithm~\ref{alg:DPsgd} is bounded as follows: 
\begin{align*}
    \Lambda(\alg)\le G\eta\sqrt{T}+\frac{G\eta T}{nm}.
\end{align*}
\end{restatable}


Finally, to prove our main result, we need the following convergence rates for SGD.
\begin{proposition}[SGD, \cite{bubeck2015convex}]
\label{prop:SGD}
Consider a convex function $f$ over a convex domain $X$.
Suppose the random initial point $x_0$ satisfies $\E[\|x_0-x^*\|]\le R^2$ where $x^*=\arg\min_{x\in X}f(x)$.
 Assume the unbiased stochastic oracle is such that $\E[\|\Tilde{g}(x)\|^2]\le \sigma^2$.
 Running gradient descent with step size $\eta$ satisfies 
 \begin{align*}
     \E\left[\frac{1}{T}\sum_{t=1}^{T}f(x_{t+1})-\min_{x^*}f(x^*)\right]\le \frac{R^2}{\eta T}+\eta \sigma^2.
 \end{align*}

Moreover, if the function $f$ is $\beta$-smooth and the unbiased stochastic oracle is such that $\E[\|\Tilde{g}(x)-\nabla f(x)\|^2]\le \sigma^2$,
then running SGD for $T$ steps with step size $\eta$ satisfies that
\begin{align*}
    \E\left[\frac{1}{T}\sum_{t=1}^{T}f(x_{t+1})-\min_{x^*}f(x^*)\right]\le (\beta+\frac{1}{\eta})\frac{R^2}{T} + \frac{\eta\sigma^2}{2}.
\end{align*}
\end{proposition}

Combining these lemmas, we are now ready to prove~\Cref{thm:general_convex}.

\begin{proof}[Proof of Theorem~\ref{thm:general_convex}]


The privacy guarantee of Algorithm~\ref{alg:DPsgd} follows from the privacy guarantee of our mean estimation procedure (Algorithm~\ref{alg:gauss_mech_gradient}), as Algorithm~\ref{alg:DPsgd} is post processing of the outputs of  Algorithm~\ref{alg:gauss_mech_gradient}.

Now, we prove utility. Let $\hat \theta = \frac{1}{T}\sum_{t\in[T]}\theta_t $ denote the output of the algorithm. We upper bound the error by splitting it to two terms: one for generalization error and empirical error,
\begin{align}
\label{eq:alg-err}
     \E\left[L_\calP(\hat \theta)-\min_{\theta^* \in \Theta}L_\calP(\theta^*) \right]
        & = \E\left[L_\calP(\hat \theta)-L_\calD(\hat \theta) \right] 
          + \E\left[L_\calD(\hat \theta) - \min_{\theta \in \Theta}L_\calD(\theta)  \right] 
          + \E\left[\min_{\theta \in \Theta}L_\calD(\theta) - \min_{\theta^* \in \Theta}L_\calP(\theta^*)  \right] \nonumber \\
        & \le \E\left[L_\calP(\hat \theta)-L_\calD(\hat \theta) \right] 
          + \E\left[L_\calD(\hat \theta) - \min_{\theta \in \Theta}L_\calD(\theta)  \right].
\end{align}
where the second inequality holds since $\E[\min_{\theta \in \Theta} L_\calD(\theta)]\le \min_{\theta^* \in \Theta}L_\calP(\theta^*)$.

For the empirical quantity (the second quantity in~\Cref{eq:alg-err}), first note that the error caused by randomized smoothing is $Gr$ (\Cref{lm:random_smooth}), hence
\begin{align*}
     \E\left[L_\calD(\hat \theta) - \min_{\theta \in \Theta}L_\calD(\theta)  \right] 
     \le \E\left[\hat L_\calD(\hat \theta) - \min_{\theta \in \Theta} \hat  L_\calD(\theta)  \right] + 2 G r.
\end{align*}     
As our algorithm basically applies noisy SGD over $\hat L_\calD$, we now use Proposition~\ref{prop:SGD} to bound the empirical error.
By Lemma~\ref{lem:unbiased} and Theorem~\ref{thm:mean-est}, we have 
\begin{align*}
    \og_t\sim_{\delta/Tnmd}\nabla \hat L_\calD(\theta_{t-1})+\zeta,
\end{align*}
where $\zeta\sim \calN(0,\frac{G^2T\log^2(Tmnd/\delta)}{mn^2\epsilon^2})$.
Hence we know the variance of the stochastic (sub)gradients we get is bounded by $O(\frac{G^2Td\log^2(Tmnd/\delta)}{mn^2\epsilon^2})$.
Moreover, we know that $\hell$ is $\frac{G\sqrt{d}}{r}$-smooth by Lemma~\ref{lm:random_smooth}.
Thus, \Cref{prop:SGD} now implies that
\begin{align*}
    \E[\hL_\calD(\hat \theta)-\min_{\theta}\hL_\calD(\theta)]\lesssim \left(\frac{G\sqrt{d}}{r}+\frac{1}{\eta }\right) \frac{R^2}{T}+\frac{\eta G^2Td\log^2(Tmnd/\delta)}{mn^2\epsilon^2}+\frac{GR\delta}{mnd},
\end{align*}
where the term $\frac{GR\delta}{mnd}$ comes from the failure probability.


Now we proceed to upper bound the generalization error (first quantity in~\Cref{eq:alg-err}). 
Combining Lemma~\ref{lm:generalization_stability} and Lemma~\ref{lem:stability}, and the assumption that the functions are $G$-Lipschitz, we get
\begin{align*}
    \E[L_\calP(\hat \theta)-L_\calD(\hat \theta)]\le G^2\eta\sqrt{T}+\frac{G^2\eta T}{nm}+\frac{GR\delta}{mnd}.
\end{align*}
Overall, combining these together and putting them back into~\Cref{eq:alg-err}, we get
\begin{align*}
    \E\left[L_\calP(\hat \theta)-\min_{\theta^* \in \Theta}L_\calP(\theta^*) \right]
    \lesssim & \frac{G\sqrt{d}R^2}{rT}+\frac{R^2}{\eta T}+\frac{\eta G^2Td\log^2(Tmnd/\delta)}{mn^2\epsilon^2}+Gr+ G^2\eta\sqrt{T}+\frac{G^2\eta T}{nm}+\frac{GR\delta}{mnd}.
\end{align*}

Optimizing the above parameters by setting $r=\frac{d^{1/4}R}{\sqrt{T}}$, $\eta=\frac{R}{G}\cdot\min\{\frac{\sqrt{m}n\epsilon}{T\sqrt{d\log^2(Tmnd/\delta)}},\frac{1}{T^{3/4}},\sqrt{nm}/T\}$, we get 
\begin{align*}
    & \E\left[L_\calP(\hat \theta)-\min_{\theta^* \in \Theta}L_\calP(\theta^*) \right]
    \lesssim  GR \cdot \left(\frac{d^{1/4}}{\sqrt{T}}+\frac{1}{T^{1/4}}+\frac{\sqrt{d\log^2(Tmnd/\delta)}}{\sqrt{m}n\epsilon}+\frac{1}{\sqrt{nm}} \right).
\end{align*}
By setting $T=O(m^2n^2+mn\sqrt{d})$, we have
\begin{align*}
    & \E\left[L_\calP(\hat \theta)-\min_{\theta^* \in \Theta}L_\calP(\theta^*) \right]
    \lesssim  GR\cdot \left(\frac{1}{\sqrt{nm}}+\frac{\sqrt{d\log^2(ndm/\delta)}}{n\epsilon\sqrt{m}} \right),
\end{align*}
which completes the proof.
\end{proof}


\subsection{Implication for Strongly convex functions}
Building on our optimal algorithm for the convex setting, in the section, we proceed to obtain optimal rates for the strongly convex case using the localization framework~ \cite{FeldmanKoTa20}. The idea is to iteratively run~\Cref{alg:DPsgd} for $\log \log (mn)$ rounds, where at each round, we run it with improved parameters. We present the details in~\Cref{alg:sgd_strongly_convex}, and defer the full proof to the supplement with detailed parameter settings therein.

\begin{algorithm2e}
\caption{User-level DP-SCO for strongly convex functions}
\label{alg:sgd_strongly_convex}
{\bf Input:} Dataset $\calD = (Z_1,\dots,Z_n) \in (\calZ^m)^n$, privacy parameters $(\epsilon,\delta)$, initial point $\theta_0$\;
Set $k=\lceil \log\log mn\rceil$\;
Divide $\calD$ into $k$ disjoint datasets $\{\calD_i\}_{i\in[k]}$, where $\calD_i$ is of size $n_i:=n/2^{k+1-i}$\;
\For{$i=1,\cdots,k$}
{
Run Algorithm~\ref{alg:DPsgd} with $\calD_i,\epsilon,\delta,\theta_{i-1},r_i,T_i,\eta_i,\tau_i,\hat{R}_i$ as inputs, and get its output $\theta_i$\;
}
{\bf Output:} $\hat \theta = \theta_k$\;
\end{algorithm2e}

\begin{restatable}[Strongly convex case]{theorem}{stronglyconvex}
    \label{thm:strongly_convex}
  For $0<\epsilon<10,0<\delta<1$, \Cref{alg:sgd_strongly_convex} is user-level $(\epsilon,\delta)$-DP. Under the same assumptions as in Theorem~\ref{thm:general_convex}, additionally assuming that $n>\frac{\log (mdn)\log(mdn/\delta)}{\epsilon}$ and the functions are $\mu$-strongly convex, then with proper parameter settings,
  Algorithm~\ref{alg:sgd_strongly_convex} outputs $\hat \theta$ such that 
  \begin{align*}
  & \E\left[ L_{\calP}(\hat \theta) - \min_{\theta^\star \in \Theta} L_{\calP}( \theta^\star) \right]   \le O\left(\frac{G^2}{\mu} \cdot \left(\frac{1}{nm}+\frac{d\log^2(ndm/\delta)}{n^2m\epsilon^2} \right) \right).
  \end{align*}
\end{restatable}

\section{Conclusion}
In this work, we have studied user-level DP-SCO and proposed new efficient algorithms that obtain near-optimal rates even in the non-smooth setting. There remain open questions in this domain. First, our rates are optimal up to logarithmic factors and we leave it for future work to improve these factors. Moreover, our algorithms require the number of rounds $T \ge n^2 m^2 \cdot \min(1, n^2/d)$, and it remains open whether there is a more efficient algorithm. In particular, are there linear time algorithms for user-level DP-SCO in the smooth setting, similar to the item-level setting where such results are known~\cite{FeldmanKoTa20}? 




\addcontentsline{toc}{section}{References}
\bibliographystyle{alpha}
\bibliography{ref}
\newpage

\appendix



\section{Missing Proofs in Section~\ref{sec:mean_est}}
\subsection{Proof of Lemma~\ref{lem:bound_l1_f}}

\boundlonef*

\begin{proof}
    Without loss of generality, let $\calD = (Z_1, Z_2,\dots,Z_n)$  and $\calD' = (Z_1',Z_2, \dots,Z_n)$ differ in the first user. Note that $f_{i,j}$ has sensitivity $1$ for $j \neq 1$, hence $|p_{i,j} - p'_{j,j}| \le 1/n$ for all $j \neq 1$. Moreover, $|p_{i,1} - p'_{i,1}| \le 1$. Therefore, $\lone{p_i-p'_i} \le 2$.
\end{proof}

\subsection{Proof of Lemma~\ref{lem:coupling_over_multi_ber}}
\couplingovermultiber*
\begin{proof}
We construct the coupling by considering each coordinate separately.
Let $p_i$ and $p_i'$ be the $i$-th coordinate of $p$ and $p'$ respectively.
Consider $i$-th coordinate, without losing generality, let $p_i\ge p_i'$.
Then, we set
\begin{align*}
    (x_i,y_i)=
    \begin{cases}
        (1,1),  & \text{w.p. } p_i' \\
        (1,0), & \text{w.p. } p_i-p_i'\\
        (0,0), & \text{w.p. } 1-p_i
    \end{cases}
\end{align*}
And coordinates are independent of each other.
We draw $(x,y)$ from the coupling $\Gamma$, and set $X_i=1$ if $x_i=y_i$ and, $X_i=0$ otherwise.
Hence we know $\{X_i\}$ are independent Bernoulli random variables such that $\E[X_i]=|p_i-p_i'|$.
By Lemma~\ref{lm:Chernoff}, we know
\begin{align*}
    &~\Pr[\|x-y\|_1\ge O(\log(1/\zeta))]
    =\Pr[\sum_{i}X_i\ge O(\log(1/\zeta))]\le \zeta.
\end{align*}
This completes the proof.
\end{proof}

\subsection{Proof of Lemma~\ref{lem:coupling_close_gi}}

Recall that $E$ corresponds to the absolute values of the Laplacian noise used in $\AboTh$ are bounded.
Define $E'$ to be the event w.r.t. input $\calD'$.

\couplingclosegi*
\begin{proof}
If $a_i=a_i'=\bot$, then both $g_i$ and $g_i'$ will be $\mathbf{0}$.

Consider the non-trivial case when $a_i=a_i'=\top$.
As $\qcon_i(\calD,\tau)>\frac{2n}{3}$, we know there exists $Z^*\in\calD$ such that $\sum_{Z\in\calD}\ind(\|q_i(Z^*)-q_i(Z)\|\le \tau)\ge \frac{2n}{3}$.
Let $H_i=\{Z\in\calD:\|q_i(Z)-q_i(Z^*)\|\le \tau\}$ be the set of users whose queried values are close to $Z^*$.
We know $H_i\subset S_i$.
Moreover, we can argue for any $Z\in S_i$, $\|q_i(Z)-q_i(Z^*)\|\le 4\tau$.
The same argument holds for $\calD'$, that is there exists $Z'^*\in\calD'$, such that $H_i'\subset S_i'$  and for any $Z\in S_i',\|q_i(Z)-q_i(Z'^*)\|\le 4\tau$.

We know $\|q_i(Z^*)-q_i(Z'^*)\|\le 2\tau$, as there exists $Z$ in $\calD\cap \calD'$ such that $\|q_i(Z^*)-q_i(Z)\|\le \tau$ and $\|q_i(Z'^*)-q_i(Z)\|\le \tau$. 
Hence for any point $Z_1,Z_2\in S_i\cup S_i'$, $\|q_i(Z_1)-q_i(Z_2)\|\le 10\tau$.

Note that $g_i=\frac{1}{|S_i|}\sum_{Z\in S_i}q_i(Z)$ and $g_i'=\frac{1}{|S_i'|}\sum_{Z\in S_i'}q_i(Z)$.
By Lemma~\ref{lem:bound_l1_f} and Lemma~\ref{lem:coupling_over_multi_ber}, we know there exists a Coupling $\Gamma_i$ over $S_i$ and $S_i'$ such that if we draw $(S,S')$ from $\Gamma_i$, with probability at least $1-\zeta$, we have
\begin{align*}
    \|S-S'\|_0\lesssim \log(1/\zeta).
\end{align*}
Assume $|S'| \ge |S|$ without loss of generality and let $Z_0 \in S$. Note that we have
\begin{align*}
    &~~~ \|g_i-g_i'\|_2\\
    & = \left\|\frac{1}{|S|}\sum_{Z\in S}q_i(Z)-\frac{1}{|S'|}\sum_{Z\in S'}q_i(Z)\right\|_2 \\
    & = \frac{1}{|S'|} \left\|\frac{|S'|}{|S|}\sum_{Z\in S}q_i(Z)- \sum_{Z\in S'}q_i(Z)\right\|_2 \\
    & = \frac{1}{|S'|} \left\|\frac{|S'| - |S|}{|S|}\sum_{Z\in S}q_i(Z) + \sum_{Z\in S}q_i(Z) - \sum_{Z\in S'}q_i(Z)\right\|_2 \\
    & = \frac{1}{|S'|} \left\|\frac{|S'| - |S|}{|S|}\sum_{Z\in S}q_i(Z) + \sum_{Z\in S \setminus S'}q_i(Z) - \sum_{Z\in S' \setminus S}q_i(Z)\right\|_2 \\
    & \le  \frac{1}{|S'|} \left\|\frac{|S'| - |S|}{|S|}\sum_{Z\in S}q_i(Z) + \sum_{Z\in S \setminus S'}q_i(Z) - |S' \setminus S| \cdot q_i(Z_0)\right\|_2  + \frac{1}{|S'|} \left\| |S' \setminus S| \cdot q_i(Z_0)  - \sum_{Z\in S' \setminus S}q_i(Z)\right\|_2 \\
    & \stackrel{(i)}{=}  \frac{1}{|S'|} \left\|\frac{|S'| - |S|}{|S|}\sum_{Z\in S}(q_i(Z)-q_i(Z_0)) + \sum_{Z\in S \setminus S'}(q_i(Z)-q_i(Z_0)) \right\|_2  + \frac{1}{|S'|} \left\| \sum_{Z\in S' \setminus S}(q_i(Z_0) - q_i(Z))\right\|_2 \\
    & \stackrel{(ii)}{\le}    \frac{10\tau}{|S'|} \cdot \left( (|S'| - |S|) + |S \setminus S'| + |S' \setminus S| \right) \\
    & \stackrel{(iii)}{\lesssim}   \frac{\tau \log(1/\zeta)}{n}.
\end{align*}
where $(i)$ follows since $|S'| - |S| + |S\setminus S'| = |S'\setminus S|$, and $(ii)$ follows since $\max_{Z_1,Z_2\in S\cup S'}\|q_i(Z_1)-q_i(Z_2)\|_2 \le 10 \tau$, and $(iii)$ follows since $\|S-S'\|_0\lesssim \log(1/\zeta)$ and hence $|S'| - |S| + |S'\setminus S|+|S\setminus S'|\lesssim \log(1/\zeta)$.

This completes the proof.

\end{proof}

\subsection{Proof of~\Cref{lem:privacy_mean_est}}
\label{sec:proof-priv-mean}
\privacymeanest*
\begin{proof}
Let $\{a_i\}_{i\in T}=\{\top,\bot\}^T$ be the outputs of Algorithm~\ref{alg:mean_est_with_AT} with input $\calD$, where if $a_i=\bot$ we set $a_j=\bot$ for all $j\ge i$.
Define the $\{a_i'\}$ correspondingly with input $\calD'$.
Then by Theorem~\ref{thm:Above_Threshold}, we know $\{a_i\}$ and $\{a_i'\}$ are $(\epsilon/2,0)$-indistinguishable.

Now conditional on that $E$ and $E'$ hold.
If $a_i=a_i'=\bot$, then the algorithm halts and outputs the initial point, hence no privacy leakage.

Our proof proceeds by assuming the Gaussian noise $v_i$ we add is drawn from $\calN(0,\frac{4\tau^2T\log(1/\zeta')\log(1/\delta')}{n^2\epsilon^2})$.
Then the statement follows from setting $\zeta'$ and $\delta'$.

Under the assumption on $n\ge\frac{40\log(2T/\zeta')}{\epsilon}$, for any $b\in\{\top,\bot\}^T$,
by Lemma~\ref{lem:coupling_close_gi}, the Union Bound, we know there exists a coupling over $\{g_i\}_{i\in T}$ and $\{g_i'\}_{i\in T}$, such that for $(\{x_i\},\{y_i\})$ drawn from $\Gamma$, with probability at least $1-T\zeta'$,
\begin{align*}
    \text{ for all } i\in[T], \|x_i-y_i\|\lesssim \frac{\tau\log(1/\zeta')}{n}.
\end{align*}

By the guarantee of the Gaussian Mechanism and the composition \cite{bun2016concentrated},
we know 
\begin{align*}
    \Pr[\{g_i+\nu_i\}\in\calO\mid E,\{a_i\}=b]\le e^{\epsilon/2}\Pr[\{g_i'+\nu_i'\}\in\calO \mid E',\{a_i'\}=b]+\delta'+T\zeta',
\end{align*}
where we note that the Gaussian noise of $\{\nu_i\}$ and $\{\nu_i'\}$ are independent of the Laplacian noise we add in Algorithm~\ref{alg:mean_est_with_AT}.

To conclude, letting $\{g_i+\nu_i\}$ be the sequence of output, we have for any event $\calO$,
\begin{align*}
    \Pr[\{g_i+\nu_i\}\in \calO]= & \Pr[\{g_i+\nu_i\}\in \calO\mid E]\Pr[E]+\Pr[\{g_i+\nu_i\}\in \calO\mid \neg E]\Pr[\neg E]\\
    \le & ~ \Pr[\{g_i+\nu_i\}\in \calO\mid E]\Pr[E]+\zeta'\\
    =& ~ \sum_{b\in\{\top,\bot\}^T} \Pr[\{g_i+\nu_i\}\in \calO\mid E,\{a_i\}=b]\Pr[E,\{a_i\}=b]+\zeta'\\
    \le& ~ \sum_{b\in\{\top,\bot\}^T} e^{\epsilon/2}(\Pr[\{g_i'+\nu_i'\}\in \calO\mid E',\{a_i'\}=b]+\delta')\Pr[E,\{a_i\}=b]+(T+1)\zeta'\\
    \le & ~ \sum_{b\in\{\top,\bot\}^T} e^{\epsilon/2}\Pr[\{g_i'+\nu_i'\}\in \calO\mid E',\{a_i'\}=b]\Pr[E,\{a_i\}=b]+(T+1)\zeta'+e^{\epsilon/2}\delta'.
\end{align*}

Note that the randomness of $\{a_i\}$ and whether $E$ holds comes from the Laplacian variables we draw.
By the privacy guarantee of $\AboTh$, for any $b\in\{\top,\bot\}^T$, we have
\begin{align*}
    \Pr[\{a_i\}=b]\le e^{\epsilon/2}\Pr[\{a_i'\}=b].
\end{align*}
It is not hard to observe that
\begin{align*}
    \Pr[\{a_i\}=b,E]\le e^{\epsilon/2}\Pr[\{a_i'\}=b,E']+e^{\epsilon/2}\zeta'.
\end{align*}

Hence 
\begin{align*}
    &~ \Pr[\{g_i+\nu_i\}\in \calO]\\
    \le & ~ \sum_{b\in\{\top,\bot\}^T} e^{\epsilon/2}\Pr[\{g_i'+\nu_i'\}\in \calO\mid E',\{a_i'\}=b]\Pr[E,\{a_i\}=b]+(T+1)\zeta'+e^{\epsilon/2}\delta'\\
    \le & ~ \sum_{b\in\{\top,\bot\}^T} e^{\epsilon}\Pr[\{g_i'+\nu_i'\}\in \calO\mid E',\{a_i'\}=b]\Pr[E',\{a_i'\}=b]+(T+1+e^{\epsilon})\zeta'+e^{\epsilon/2}\delta'
\end{align*}
Setting $\zeta'=\frac{\delta}{2(e^{\epsilon}+1+T)}$ and $\delta'=\frac{\delta}{2e^{\epsilon/2}}$, we get the Noise scale as stated in the pseudo-code of Algorithm~\ref{alg:gauss_mech_gradient} and complete the proof.
\end{proof}



\section{Missing Proof in Section 4}

\subsection{Proof of Lemma~\ref{lem:concentrated_grd}}
\concentratedgrd*
\begin{proof}
The lemma follows from the concentration of Norm Subgaussian random variables (Lemma~\ref{lem:concen_nSG}).
Specifically, we know for each $z_{i,j}\in Z_i$, $\E\nabla\hat{\ell}(\theta+y_j;z_{i,j})-\nabla\hL_\calP(\theta)=0$, and $\|\nabla\hat{\ell}(\theta+y_j;z_{i,j})-\nabla\hL_\calP(\theta)\|\le 2G$, which implies  $\nabla\hat{\ell}(\theta+y_j;z_{i,j})-\nabla\hL_\calP(\theta)$ is zero-mean and $\nSG(2G)$.
The statement follows.
\end{proof}

\subsection{Proof of Lemma~\ref{lem:concentrated_theta_t}}
\concentratedthetat*
\begin{proof}
It suffices to prove that for each $t\in[T]$, $\{\nabla\hell(\theta_t;Z_i)\}_{Z_i\in\calD}$ is $(\tau,e^{2\epsilon}\gamma+\frac{\delta}{2Tmnd})$-concentrated.
Note that by Theorem~\ref{thm:mean-est} and the parameter settings in the precondition, Algorithm~\ref{alg:DPsgd} is user-level $(\epsilon,\frac{\delta}{2Tnmd})$-DP.
Then there exists an $(2\epsilon,0)$-DP $\alg'$ such that $d_{TV}(\alg(\calD),\alg'(\calD))\le \delta/2Tmnd$ by Lemma~\ref{lem:approximateDP_to_pureDP}.
Let $\{\theta_t'\}_{t\in[T]}$ be the output of $\alg'(\calD)$.
It suffices to show for any $t\in[T],\{\nabla\hell(\theta_t';Z_i)\}_{Z_i\in\calD}$ is $(\tau,e^{2\epsilon}\gamma)$-concentrated.


Let $f_{Z_i}(Z)$ be the density of $Z_i=Z$ and and $f_{Z_i}(Z\mid \theta_t'=\theta)$ be the density conditional on $\theta_t'=\theta$.
Similarly, we let $f_{\theta_t'}(\theta)$ and $f_{\theta_t'}(\theta\mid Z_i=Z)$ be the (conditional) density of $\theta_t'$.
For any $\theta,Z$, we have
\begin{align*}
    \frac{f_{Z_i}(Z\mid \theta_t'=\theta)}{f_{Z_i}(Z)}=\frac{f_{\theta_t'}(\theta\mid Z_i=Z)}{f_{\theta_t'}(\theta)}\le e^{2\epsilon},
\end{align*}
where the last inequality comes from the privacy guarantee of $\alg'$.

One has
\begin{align*}
    & ~\Pr_{Z_i,\theta_t'}\Big[\|\nabla\hell(\theta_t';Z_i)-\nabla\hL_\calP(\theta_t')\|\ge\tau\Big]\\
    =& \int\int f_{\theta_t'}(\theta)f_{Z_i}(Z\mid\theta_t'=\theta)\ind(\|\nabla\hell(\theta;Z)-\nabla\hL_\calP(\theta)\|\ge\tau)\d Z \d \theta\\
    \le & e^{2\epsilon}\int\int f_{\theta_t'}(\theta)f_{Z_i}(Z)\ind(\|\nabla\hell(\theta;Z)-\nabla\hL_\calP(\theta)\|\ge\tau)\d Z\d \theta.
\end{align*}
Note that for any $\theta$, we have
\begin{align*}
    \int f_{Z_i}(Z)\ind(\|\nabla\hell(\theta;Z)-\nabla\hL_\calP(\theta)\|\ge\tau)\d Z\le \gamma/n.
\end{align*}

Then by union bound, we know $\{\nabla\hell(\theta_t';Z_i)\}_{Z_i\in \calD}$ is $(\tau,e^{2\epsilon}\gamma)$-concentrated which completes the proof as $d_{TV}(\alg(\calD),\alg'(\calD))\le \delta/2Tmnd$.
\end{proof}

\subsection{Proof of Lemma~\ref{lem:stability}}
\stability*

\begin{proof}
We use~\Cref{lemma:sgd-stab} to upper bound the stability of our algorithm.
As we are using fixed step sizes $\eta_t=\eta$, \Cref{lemma:sgd-stab} implies that 
\begin{align*}
     \Lambda(\alg) 
     & \le 2G\sqrt{\sum_{t\in[T-1]}\eta_t^2}+2\sum_{t\in[T-1]}\eta_ta_t \\
     & \le 2G\eta\sqrt{T}+2\eta\sum_{t\in[T-1]} a_t
\end{align*}
Thus it suffices to upper bound $a_t$ for all $t\in[T]$.

By Lemma~\ref{lem:concentrated_theta_t}, we know for all $t\in[T]$, $\{\nabla\hell(\theta_t;Z_i)\}_{Z_i\in \calD}$ is $(\tau,\gamma')$-concentrated for
\begin{align*}
    \tau=\frac{G\log(nd/\gamma)}{\sqrt{m}}, \gamma'=T(e^{2\epsilon}\gamma+\frac{\delta}{2Tmnd}).
\end{align*}
Then by Theorem~\ref{thm:mean-est} and Lemma~\ref{lem:unbiased}, we know 
\begin{align*}
    \overline{g}_t\sim_{2\gamma'}  \frac{1}{nm}\sum_{Z_i\in \calD}\sum_{z_{i,j}\in Z_i}\nabla\hat{\ell}(\theta_t+y_j;z_{i,j})+\nu,
\end{align*}
where $\nu$ is Gaussian noise independent of the data.
Thus we have $a_t\le\frac{G}{nm}$. Setting $\gamma=\frac{\delta}{2Te^{2\epsilon}}$ completes the proof.
\end{proof}

\subsection{Proof of Theorem~\ref{thm:strongly_convex}}
\stronglyconvex*

\begin{proof}
Let $L_\calP^* = \min_{\theta^* \in \Theta}L_\calP(\theta^*) $, $\Delta_i:=\E[L_\calP(\theta_i)-L_\calP^*]$ and $R_i^2:=\E[\|\theta_i-\theta^*\|^2]$.
Due to the strong convexity, we know $\frac{1}{2}\mu R_i^2\le \Delta_i$.

Let $C>2$ be the constant hidden in the population loss bound in Theorem~\ref{thm:general_convex}.
For $i\ge0$, define $E_i:=\frac{4C^2G^2}{\mu}(\frac{1}{n_im}+\frac{d\log^2(n_idm/\delta)}{n_i^2\epsilon^2m})$ and we know $E_{i}/E_{i+1}\le 4$.
Define $D_i=16E_i\sqrt[2^{i}]{\frac{2G^2}{\mu}\cdot\frac{1}{16E_0}}$.
By the definition, we know
\begin{align*}
    \frac{D_{i+1}}{16E_{i+1}}=\sqrt[2^{i}]{\frac{2G^2}{\mu}\cdot\frac{1}{16E_0}} \le \sqrt{\frac{D_i}{16E_i}},\\
    \sqrt{D_iE_{i+1}}=4\sqrt{E_iE_{i+1}}\sqrt[2^{i}]{\frac{2G^2}{\mu}\cdot\frac{1}{4E_1}}\le 16E_{i+1}\sqrt[2^{i+1}]{\frac{2G^2}{\mu}\cdot\frac{1}{4E_1}}  = D_{i+1}.
\end{align*}

Hence by setting $k\ge\log\log (D_1/(16E_{1}))$, then $\frac{D_k}{16E_k}\le 2$.
Note that $E_0\ge \frac{4C^2G^2}{\mu nm}$, and $D_0=\frac{2G^2}{\mu}$.
We get $\frac{D_0}{16E_0}\le mn$ and setting $k=\log\log(mn)$ is large enough to get $D_k\le 32E_k$.
Note that $\Delta_0\le \frac{2G^2}{\mu}$ and $R_0\le \frac{2G}{\mu}$ by the strong convexity and assumption on being Lipschitz.

For $j\ge1$, set $\hat{R}_j=\sqrt{2D_{j-1}/\mu}$,  
$r_j=\frac{d^{1/4}\hat R_j}{\sqrt{T_j}},\eta_j=\frac{\hat R_j}{G}\cdot\min\{\frac{\sqrt{m}n_j\epsilon}{T_j\sqrt{d\log^2(mn_jd/\delta)}},\frac{1}{T_j^{3/4}},\frac{\sqrt{n_jm}}{T_j}\}$, $\tau= \frac{G\log(n_jdme^\epsilon T_j/\delta)}{\sqrt{m}}$ and $T_j=O(m^2n_j^2+mn_j\sqrt{d})$.
As $n_j\ge n/\log(nm)\ge \frac{\log(mdn/\delta)}{\epsilon}$ by the precondition, $R_0\le \hat{R}_1=\frac{2G}{\mu}$, by Theorem~\ref{thm:general_convex} and our parameter setting, recursively we know
\begin{align*}
    \Delta_{j}\le &CG\hat R_j \cdot (\frac{1}{\sqrt{n_jm}}+\frac{\sqrt{d\log^2(n_jdm/\delta)}}{n_j\epsilon \sqrt{m}})\\
    \le & CG\sqrt{2D_{j-1}/\mu} \cdot (\frac{1}{\sqrt{n_jm}}+\frac{\sqrt{d\log^2(n_jdm/\delta)}}{n_j\epsilon \sqrt{m}})\\
    \le & CG\sqrt{2D_{j-1}/\mu} \cdot \sqrt{\frac{\mu E_j}{2C^2G^2}}\\
    \le & \sqrt{D_{j-1}E_j}\le D_{j},
\end{align*}
where we used $\sqrt{a}+\sqrt{b}\le \sqrt{2(a+b)}$ for $a,b>0$.
We know $R_i\le \sqrt{\frac{2\Delta_i}{\mu}}\le \sqrt{\frac{2D_i}{\mu}}=\hat{R}_{i+1}$ recursively as well.

After $k$-iteration,
we have
\begin{align*}
    \E[L_\calP(\theta_k)-L_\calP^*]
    & =\Delta_k\le D_k \le ~ 32E_k=O(\frac{G^2}{\mu}(\frac{1}{nm}+\frac{d\log^2(ndm/\delta)}{n^2\epsilon^2m})).
\end{align*}
The statement follows.
\end{proof}
\end{document}